\documentclass[letterpaper, 10 pt, conference]{ieeeconf}
\IEEEoverridecommandlockouts 
\usepackage{salp_macros}
\usepackage{float}

\usepackage{amsthm}
\newtheorem{theorem}{Theorem}

\usepackage[dvipsnames]{xcolor}
\usepackage[normalem]{ulem}

\usepackage[colorinlistoftodos]{todonotes}  

\makeatletter
\let\NAT@parse\undefined
\makeatother
\usepackage{hyperref}  

\title{\LARGE \bf Modeling, Observability, and Inertial Parameter Estimation of a Planar Multi-Link System with Thrusters*}

\author{Nicholas B. Andrews$^{1}$ and Kristi A. Morgansen$^{1}$
\thanks{*This work was supported in part by ONR Award N00014-23-1-2171}%
\thanks{$^{1}$Department of Aeronautics and Astronautics, University of Washington, Seattle, WA 98195, USA  
{\tt\small [{\href{mailto:nian6018@uw.edu}{nian6018}, \href{mailto:morgansn@uw.edu}{morgansn}}]@uw.edu}}
}

\begin{document}

\maketitle
\thispagestyle{empty}
\pagestyle{empty}

\begin{abstract}
This research provides a theoretical foundation for modeling and real-time estimation of both the pose and inertial parameters of a free-floating multi-link system with link thrusters, which are essential for safe and effective controller design and performance. First, we adapt a planar nonlinear multi-link snake robot model to represent a planar chain of bioinspired salp robots by removing joint actuators, introducing link thrusters, and allowing for non-uniform link lengths, masses, and moments of inertia. Second, we conduct a nonlinear observability analysis of the multi-link system with link thrusters, proving that the link angles, angular velocities, masses, and moments of inertia are locally observable when equipped with inertial measurement units and operating under specific thruster conditions. The analytical results are demonstrated in simulation with a three-link system.
\end{abstract}

\section{INTRODUCTION}
The specific work here is motivated by advancements in underwater robotics, where remotely operated vehicles (ROVs) and autonomous underwater vehicles (AUVs) are widely used for inspection, mapping, and intervention tasks. In particular, this research is directed towards the next generation of small, flexible, and minimally capable underwater vehicles bioinspired by salps \cite{Villanueva2011-fj, Gatto2020-cn}—gelatinous marine organisms roughly the size and shape of a soda can that maneuver using jet propulsion by pumping water through their bodies. A distinctive feature of both salps and their bioinspired multi-link robotic counterparts is their ability to physically connect, forming long, flexible chains—often consisting of dozens of units—that dynamically reshape into structures such as spirals, rings, and arcs. 
However, the challenges of modeling, accurately measuring, and controlling the state of these systems remain an open area of research.

Observability is a necessary property for any control system because it determines whether an estimator can uniquely reconstruct the system's state from measurements, which is essential for safe and efficient operation. In particular, while physical properties such as mass and inertia can be measured in a laboratory setting for minimally configurable rigid-body systems, they become significantly more challenging and costly to measure for highly configurable and soft systems. For example, in underwater vehicles, added mass—the resistance due to accelerating through a fluid—is typically measured by towing the vehicle through the water at different speeds. However, this approach is infeasible for a configurable system, as its size and shape are time-varying, potentially at the same rate as system motion. Instead of relying on pre-deployment experiments, enabling a system to estimate both its pose (position and orientation) and inertial properties in real-time is desirable.

Prior work on rigid-body inertial parameter estimation has leveraged nonlinear observability analyses, including studies on six-degree-of-freedom aircraft~\cite{Boyacioglu2023-mv, Sundquist2024-ot}. These analyses show that certain inertial parameters are observable with a single inertial measurement unit (IMU) and can be estimated in real time under specific external force conditions. A similar study in~\cite{Heinemann2025-le} derived analytical conditions for estimating a vehicle’s roll moment of inertia and demonstrated them through experiments. For multi-link systems, \cite{An1985-nq} proposed an algorithm to estimate the mass, center of mass (CM), and moment of inertia of each link in a fixed-base manipulator.


This work adapts a multi-link snake robot model to analyze a chain of bioinspired salp robots, with the key distinction lying in their actuation: joint actuators versus thrusters. Previous research has developed models, controllers, and gait analyses for both terrestrial and underwater snake robots~\cite{Liljeback2012-fs, Kelasidi2014-fv}, and has explored their observability and state estimator design~\cite{Rollinson2011-ti, Tully2011-ka}. The main contribution of this work is a nonlinear observability analysis of the adapted multi-link model equipped with IMUs, demonstrating that the link angles, angular velocities, masses, and moments of inertia are locally observable under specific thruster conditions. This analysis complements the controllability study in~\cite{Yang2025-ja}, which employed a geometric mechanics framework for controllability and gait design, validated on \emph{LandSalp}, a representative three-link wheeled robot.

The structure of this paper is as follows: Section \ref{sec:model} introduces the salp-inspired nonlinear multi-link model with link thrusters. In Section \ref{sec:meas}, the measurement function for an IMU is derived from the general equation governing the acceleration of a point in a rotating reference frame relative to an inertial frame. Section \ref{sec:obsv} provides a brief overview of Lie derivatives and nonlinear observability. The main contribution of this work is presented in Section \ref{sec:analysis}, where a nonlinear observability analysis is presented, and the conditions for observability are derived. The observability conditions and state tracking performance are then demonstrated for a simulated three-link system in Section \ref{sec:sim} using a Kalman filter. Finally, Section \ref{sec:conclusion} summarizes the findings and future work.

\section{MULTI-LINK SYSTEM MODEL} \label{sec:model}
\begin{table*}[t]
    \vspace{0.2cm}
    \centering
    \begin{tabular}{|c|l|c|c|} 
        \hline
        \textbf{Variable} & \textbf{Description} & \textbf{Vector} & \textbf{Matrix} \\ \hline
        \(N\) & Number of links & -- & -- \\ \hline
        \(l_i\) & Half-length of link $i$ & $\mathbf{l} \in \realpos{N}$ & $\lengthmat=\diag{\mathbf{l}} \in \diagpos{N}$ \\ \hline
        \(m_i\) & Mass of link $i$ & $\mass = \mtx{\frac{1}{m_1} & \ldots & \frac{1}{m_N}} \in \realpos{N}$ & $\massmat =\diag{\mass} \in \diagpos{N}$ \\ \hline 
        \(j_i\) & Moment of inertia of link $i$ & $\inertia \in \realpos{N}$ & $\mathbf{J}=\diag{\inertia} \in \diagpos{N}$ \\ \hline
        \(\theta_i\) & CCW angle from the inertial \(x\)-axis to link \(i\) \(x\)-axis & \(\anglink \in \reals{N}\) & -- \\ \hline
        \(x_i, y_i\) & Inertial position of the CM of link \(i\) & \(\cmlinkx, \cmlinky \in \reals{N}\) & -- \\ \hline
        \(p_x,  p_y\) & Inertial position of the CM of the multi-link system & \(\cmrobot \in \reals{2}\) & --\\ \hline
        \(u_i\) & Thruster force on link $i$ & \(\control \in \reals{N}\) & --\\ \hline
        \(\psi_i\) & CCW angle from the link \(i\) \(x\)-axis to thrust vector \(u_i\) & \(\angthrust \in \reals{N}\) & --\\ \hline
        \(f_{x,i},  f_{y,i}\) & Net inertial frame $x$ and $y$ forces on link \(i\) & \(\forces_{x}, \forces_{y} \in \reals{N}\) & --\\ \hline
    \end{tabular}
    \caption{Variables and their associated vectors and matrices \cite{Liljeback2012-fs}.} \label{table:defs}
    \vspace{-0.5cm}
\end{table*}

The multi-link system model presented in this section is adapted from and inspired by the underwater multi-link snake robot model in \cite{Liljeback2012-fs, Kelasidi2014-fv}. Each ``link'' in our model represents an individual salp robot equipped with a jet thruster and an independent electronics and sensor suite, capable of forming chains of arbitrary length and configuration. Our multi-link model differs from the snake model in the following key aspects:
\begin{enumerate}
    \item Joints are unactuated
    \item Controlled thruster forces are applied at the CM of each link
    \item Mass, length, and moment of inertia are not assumed to be identical across links
\end{enumerate}
For brevity, we do not present the full derivation of the adapted model here; however, a detailed derivation of the snake model can be found in \cite{Liljeback2012-fs}.

The system state vector is defined as
\begin{equation}
    \state = \begin{bmatrix} \anglink & \cmrobot & \dot{\anglink} & \dot{\cmrobot} \end{bmatrix} \in \reals{2N+4}
    \end{equation}
where $\theta_i$ denotes the angle of the link \(i\) \(x\)-axis relative to the inertial \(x\)-axis, with counterclockwise (CCW) considered positive. The vector \(\cmrobot\) represents the \(x\)-\(y\) position of the multi-link system’s CM. We define the following sets of \(N \times N\) matrices: symmetric matrices \(\sym{N}\), symmetric positive definite matrices \(\pdef{N}\), skew-symmetric matrices \(\skewsym{N}\), diagonal matrices \(\diagmat{N}\), and diagonal positive definite matrices \(\diagpos{N}\). Table \ref{table:defs} summarizes all relevant variables, including their scalar, vector, and matrix representations. Corresponding free-body diagrams for the multi-link system and a single link are shown in Fig. \ref{fig:multilink} and \ref{fig:link}. The axes of each link's body frame are also referred to as the tangent ($x_{\text{link}, i}$) and normal ($y_{\text{link}, i}$) directions, with the frame's origin located at the link's CM.
\begin{figure}
    \centering
    \includegraphics[width=\linewidth]{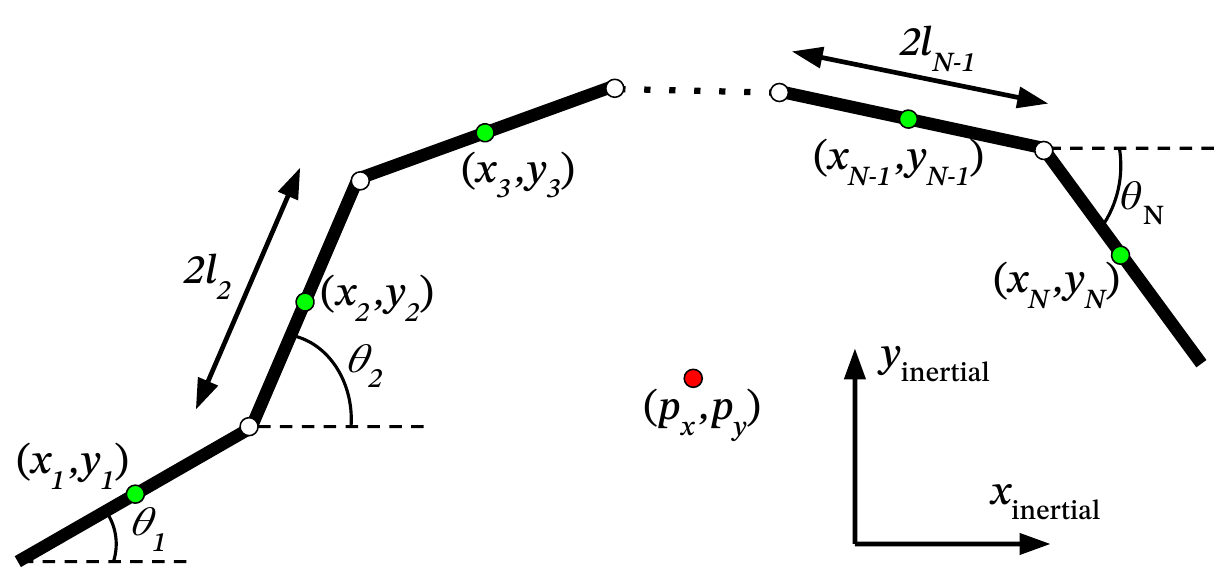} 
    \caption{Multi-link free body diagram.} \label{fig:multilink}
    \vspace{-0.5cm}
\end{figure}
\begin{figure}
    \centering
    \includegraphics[width=0.45\linewidth]{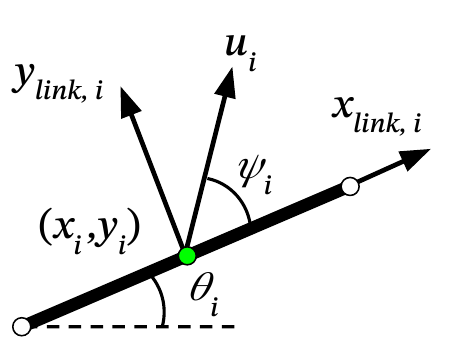} 
    \caption{Single link free body diagram.} \label{fig:link}
    \vspace{-0.5cm}
\end{figure}

The following matrices, frequently used in the derivation of the system model, are defined for conciseness:
\begin{align}
    \mathbf{A} &= \mtx{
    1 & 1 & 0 & 0 & 0 \\
    0 & \cdot & \cdot & 0 & 0 \\
    0 & 0 & \cdot & \cdot & 0 \\
    0 & 0 & 0 & 1 & 1} \in \reals{(N-1) \times N} \\
    \mathbf{D} &= \mtx{1 & -1 & 0 & 0 & 0 \\
    0 & \cdot & \cdot & 0 & 0 \\
    0 & 0 & \cdot & \cdot & 0 \\
    0 & 0 & 0 & 1 & -1} \in \reals{(N-1) \times N} \\
    \mathbf{E} &= \left[\begin{array}{cc}
    \ones{N} & \mathbf{0}_{N \times 1} \\
    \mathbf{0}_{N \times 1} & \ones{N}
    \end{array}\right] \in \reals{2 N \times 2} \\
    \mathbf{V} &= \tpose{\mathbf{A}} \inv{\left(\mathbf{D} {\massmat} \tpose{\mathbf{D}}\right)} \mathbf{A} \in \sym{N} \\
    \kmat &= \tpose{\mathbf{A}} \inv{\left(\mathbf{D} {\massmat} \tpose{\mathbf{D}} \right)} \mathbf{D} \in \reals{N \times N} \\
    \mtheta &= \mathbf{J} + \lengthmat \left( \stheta \mathbf{V} \stheta + \ctheta \mathbf{V} \ctheta \right) \lengthmat \in \pdef{N} \\
    \wmat &= \lengthmat \left(\stheta \mathbf{V C}_\theta - \ctheta \mathbf{V S}_\theta \right) \lengthmat \in \skewsym{N}.
\end{align}
Here, $\ones{N}$ is an $N$-dimensional vector of ones, $\ctheta=\diag{\cos \boldsymbol{\theta}} \in \diagmat{N}$, $\stheta=\diag{\sin \boldsymbol{\theta}} \in \diagmat{N}$, and $\diag{\cdot}$ returns a square diagonal matrix with the elements of the input vector on the diagonal.

The nonlinear dynamics of the multi-link system are:
\begin{gather}
    \dot{\state} = \mtx{\dot{\anglink} \\ \dot{\cmrobot} \\
    \inv{\mtheta} \left( -\wmat \dot{\anglink}^2 + \lengthmat \stheta \kmat {\massmat} \forces_{x} - \lengthmat \ctheta \kmat {\massmat} \forces_{y} \right) \\
    \frac{1}{m_\Sigma} \tpose{\mathbf{E}} \forces}, \label{eq:multilink}
\end{gather}
where $\forces = \mtx{\forces_{x} & \forces_{y}} \in \reals{2N}$, $\dot{\boldsymbol{\theta}}^2 = \mtx{\dot{\theta}_1^2 & \ldots & \dot{\theta}_N^2}$ and $m_\Sigma = \sum_{i=1}^N m_i$ is the total mass of the multi-link system. The net forces consist of external forces and thruster forces:
\begin{gather}
    \forces_{x} = \forces_{ext, x} + \cmatadd \control, \quad
    \forces_{y} = \forces_{ext, y} + \smatadd \control,
\end{gather}
where $\cmatadd = \diag{\cos(\bm{\theta + \angthrust})} \in \diagmat{N}$ and $\smatadd = \diag{\sin(\bm{\theta + \angthrust})} \in \diagmat{N}$. In this work, we make no specific assumptions about the external forces $\forces_{ext, x}$ and $\forces_{ext, y}$, other than that they are smooth and bounded. These forces may include hydrodynamic, gravitational, and frictional forces, depending on the application.

The CM of the multi-link system is calculated by
\begin{gather}
    \mathbf{p} = \mtx{p_x \\ p_y} = \frac{1}{m_\Sigma} \mtx{\inv{\mass} \cdot \cmlinkx \\ 
    \inv{\mass} \cdot \cmlinky},
\end{gather}
where $\inv{\mass} = \mtx{m_1 & \dots & m_N}$. The link positions and velocities can be reconstructed from the multi-link CM and link angles as follows:
\begin{gather}
    \mathbf{T} = \mtx{\mathbf{D} \\ \frac{1}{m_\Sigma} {\inv{\mass}}} \in \reals{N \times N} \\
    \cmlinkx = \inv{\mathbf{T}} \mtx{- \mathbf{A} \lengthmat \cos(\anglink) \\
    p_x}, \ \dot{\cmlinkx} = \inv{\mathbf{T}} \mtx{ \mathbf{A} \lengthmat \stheta \dot{\anglink}\\
    \dot{p_x}} \\
    \cmlinky = \inv{\mathbf{T}} \mtx{- \mathbf{A} \lengthmat \sin(\anglink) \\
    p_y}, \ \dot{\cmlinky} = \inv{\mathbf{T}} \mtx{- \mathbf{A} \lengthmat \ctheta \dot{\anglink} \\
    \dot{p_y}}.
\end{gather}

\section{IMU MEASUREMENT MODEL} \label{sec:meas}
In this section, we derive the measurement function for an IMU mounted on link $i$ of the nonlinear system~\eqref{eq:multilink}. The IMU, consisting of an accelerometer and a gyroscope, provides measurements of the local linear acceleration and angular velocity relative to the inertial frame. 
The resulting measurement function is expressed in terms of the link to which the IMU is attached and its position relative to the corresponding link frame.

\subsection{Relative Motion}
The total acceleration of a point in a rotating frame, as observed from an inertial frame, is given by  
\begin{equation}
    \begin{aligned} \label{eq:relaccel}
        \mathbf{a}_{\text{inertial}} &= \mathbf{a}_{\text{origin}} + \mathbf{a}_{\text{relative}} + 2 \boldsymbol{\omega} \times \mathbf{v} \\
        &\quad + \boldsymbol{\omega} \times (\boldsymbol{\omega} \times \mathbf{r}) + \bm{\alpha} \times \mathbf{r}.
    \end{aligned}
\end{equation}
The variables are all in $\reals{3}$ and defined as:
\begin{itemize}
    \item \( \mathbf{a}_{\text{inertial}} \): Total acceleration of the point in the inertial frame  
    \item \( \mathbf{a}_{\text{origin}} \): Acceleration of the rotating frame's origin relative to the inertial frame  
    \item \( \mathbf{a}_{\text{relative}} \): Acceleration of the point relative to the rotating frame  
    \item \( \mathbf{v} \): Velocity of the point relative to the rotating frame  
    \item \( \mathbf{r} \): Position vector of the point relative to the rotating frame's origin  
    \item \( \boldsymbol{\omega} \): Angular velocity of the rotating frame relative to the inertial frame  
    \item \( \bm{\alpha} \): Angular acceleration of the rotating frame relative to the inertial frame
\end{itemize}
Equation \eqref{eq:relaccel} will be used as the starting point for deriving the IMU measurement equation in the next subsection.

\subsection{IMU Measurement}
The measurement function for an IMU in 3D is
\begin{gather}
    \meas = h(\state) =
    \begin{bmatrix}
    \mathbf{a}_{\text{inertial}} &
    \bm{\omega}
    \end{bmatrix} \in \reals{6}.
\end{gather}
We will simplify this equation by first assuming that an IMU placed on link $i$ is fixed relative to the link frame, i.e., \mbox{$\mathbf{r} = \text{constant}$,} $\mathbf{v} = 0$, and $\mathbf{a}_{\text{relative}} = 0$. This reduces \eqref{eq:relaccel} to 
\begin{gather}
    \mathbf{a}_{\text{inertial}} = \mathbf{a}_{\text{origin}} + \boldsymbol{\omega} \times (\boldsymbol{\omega} \times \mathbf{r}) + \bm{\alpha} \times \mathbf{r}.
\end{gather}

Next, since the multi-link system \eqref{eq:multilink} is planar, the cross-product terms can be simplified. In 2D, \( \boldsymbol{\omega} \) and $\bm{\alpha}$ are perpendicular to the plane (along the \( z \)-axis) and $\mathbf{r} \in \reals{2}$, so we have:
\begin{gather}
    \mathbf{a}_{\text{inertial}} = \mathbf{a}_{\text{origin}} - \omega^2 \mathbf{r} + \alpha \mtx{0 & -1 \\ 1 & 0} \mathbf{r},
\end{gather}
where $\omega = \norm{\bm{\omega}}$ and $\alpha = \norm{\bm{\alpha}}$.

The net forces applied at the CM of link $i$ in inertial frame coordinates are $\mathbf{f}_i = \mtx{f_{x,i} & f_{y,i}} \in \reals{2}$. The acceleration at the CM of the link is proportional to the net forces applied at the CM divided by the link mass: $\mathbf{a}_{\text{origin}} = \frac{{\mathbf{f}}_i}{m_i}$. Lastly, substituting in the appropriate terms, the measurement function for an IMU on link $i$ is
\begin{align} \label{eq:meas}
    \meas_{i} &= h_i(\state) = \mtx{ \frac{{\mathbf{f}}_i}{m_i} - \dot{\theta}_i^2 \mathbf{r} + \ddot{\theta}_i \mtx{0 & -1 \\ 1 & 0} \mathbf{r} \\
    \dot{\theta}_i} \in \reals{3}.
\end{align}

\section{NONLINEAR OBSERVABILITY} \label{sec:obsv}
\label{sec:nlobsv}
For linear systems, observability has a single well-defined meaning. However, in nonlinear systems, observability can exist in varying degrees, requiring a precise definition of the specific type being considered. Below, we provide a brief review of nonlinear observability and the Lie algebraic approach for determining the observability of nonlinear systems, summarized from \cite{Hermann1977-rt, Nijmeijer1990, Mirzaei2008-kl}.

Consider the nonlinear system, \(\Sigma\), with the following process and measurement models:
\begin{gather} \label{eq:nl_system}
    \Sigma: \quad \dot{\state} = f(\state, \control), \quad
    \meas = h(\state),
\end{gather}
where \( \state(t) \in \mathbb{R}^{n} \), \( \meas(t) \in \mathbb{R}^{k} \), and \( \control(t) \in \mathcal{U} \subseteq \mathbb{R}^{m} \), with \(\mathcal{U}\) the set of admissible controls. Let \( \state(t, \state_0, \control) \) denote the solution to the initial value problem for \(\Sigma\) with initial condition \( \state(0) = \state_0 \) under the control input \( \control(t) \), and define \( \meas(t, \state_0, \control) = h(\state(t, \state_0, \control)) \).


\emph{Local observability} is defined as follows: \( \state_0 \) and \( \state_1 \) are \emph{\( V \)-indistinguishable} if, for every control \( \control \in \mathcal{U} \), the corresponding state trajectories \( \state(t, \state_0, \control) \) and \( \state(t, \state_1, \control) \) remain in \( V \subseteq \mathbb{R}^n \) over \( t \in [0, T] \) and satisfy  
\begin{gather}
    \meas(t, \state_0, \control) = \meas(t, \state_1, \control) \ \forall \ t \in [0, T].
\end{gather}

The system \(\Sigma\) is \emph{locally observable at \( \state_0 \)} if there exists a neighborhood \( W \) of \( \state_0 \) such that, for every neighborhood \( V \subset W \), indistinguishability within \( V \) implies \( \state_0 = \state_1 \). If \(\Sigma\) is locally observable at all \( \state \), then it is \emph{locally observable}. Intuitively, this definition means that \( \state_0 \) can be distinguished from nearby states within finite time and with trajectories remaining close to \( \state_0 \). 

A differential geometric approach to testing observability involves analyzing the \emph{Lie derivatives} of the output function \( h(\state) \) with respect to the process model \( f(\state) \). The zeroth through second-order Lie derivatives are given by:
\begin{align}
    \lie{h}{f}{0} &= h(\state), \\
    \lie{h}{f}{1} &= \nabla h(\state) \cdot f(\state), \\
    \lie{h}{f}{2} &= \nabla (\lie{h}{f}{1}) \cdot f(\state) = \lie{\left( \lie{h}{f}{1} \right)}{f}{1}.
\end{align}
If \( h(\state) \) is a scalar function, \( \nabla h(\state) \) is its gradient (a row vector). If \( h(\state) \) is a vector function, \( \nabla h(\state) \) is the Jacobian matrix. Higher-order Lie derivatives follow a similar recursive form.  

If the process model is \emph{control-affine}, meaning it decomposes as: $\dot{\state} = f_0(\state) + \sum_{i=1}^{m} f_i(\state) \ u_i$, where \( f_0 \) represents the drift term, then Lie derivatives can be computed separately with respect to the drift and control vector fields. For instance, a second-order mixed Lie derivative of \( h(\state) \), first along \( f_0 \) and then along \( f_1 \), is:
\begin{equation}
    \mathcal{L}_{f_1 f_0}^2 h(\state) = \mathcal{L}_{f_1}^1 (\mathcal{L}_{f_0}^1 h(\state)) = \nabla (\mathcal{L}_{f_0}^1 h(\state)) \cdot f_1(\state).
\end{equation}
Taking a Lie derivative with respect to \( f_i(\state) \) implicitly assumes \( u_i \neq 0 \). 

The \emph{observability Lie algebra} $\obsvspace{}$, or observation space, for a control-affine system $\Sigma$ is
\begin{equation}
    \obsvspace{} = \text{span}\left\{ \mathcal{L}_{f_i}^c h(\state) \mid c \in \mathbb{N}_0, \; i = 0, \dots, m \right\}.
\end{equation}
Here, \( f_i \) can be any Lie derivative combination of the control-affine process model functions. Observability is then determined using the following rank condition on the \mbox{Jacobian of $\obsvspace{}$:}
\begin{theorem}
    The nonlinear system is locally observable if \textnormal{$\rank{\obsv{}} = n$} \cite{Hermann1977-rt}.
\end{theorem}
There is no universal method for constructing \( \obsvspace{} \), but in practice, sequentially computing Lie derivatives along well-chosen combinations of the process model functions often yield good results. If any set of Lie derivatives satisfies the rank condition, then the system is locally observable.

\section{OBSERVABILITY ANALYSIS} \label{sec:analysis}
Our goal is to establish the observability conditions for the state of the multi-link system \eqref{eq:multilink} along with a set of inertial parameters. For this analysis, we make the following assumptions:
\begin{enumerate}
    \item The length of each link ($l_i$) is known
    \item The mass ($m_i$) and inertia ($j_i$) of each link are treated as independent variables
    \item Non-zero thruster inputs ($u_i \neq 0$)
    \item One IMU placed at the CM of each link
\end{enumerate}

The inertial parameters with which this analysis is concerned are the mass and inertia of each link. In practice, dead reckoning enables recovery of the multi-link system’s position and velocity from an initial estimate by integrating IMU measurements over time. Consequently, these quantities are omitted from the state in the observability analysis. The augmented state vector and process model are:
\begin{gather}
    \stateip = \mtx{\anglink & \dot{\anglink} & \mass & \inertia} \in \reals{4N} 
    \label{eq:aug_x} \\
    \dstateip = \mtx{\dot{\anglink} \\
    \inv{\mtheta} \left( -\wmat \dot{\anglink}^2 + \lengthmat \stheta \kmat {\massmat} \forces_{x} - \lengthmat \ctheta \kmat {\massmat} \forces_{y} \right) \\
    \zeros{2N}} \label{eq:aug_f}.
\end{gather}
The dynamics can be written in a control-affine form with the control input function for thruster $u_i$ as
\begin{gather}
    f_i(\stateip) = \mtx{\zeros{N} \\
    \inv{\mtheta} \lengthmat \left( \stheta \kmat \cmatadd - \ctheta \kmat \smatadd \right) \massmat \basis{i} \\
    \zeros{2N}}.
\end{gather}
By placing each IMU at the origin of its link frame ($\mathbf{r} = \mathbf{0}$), the measurement equation \eqref{eq:meas} for link $i$ simplifies to
\begin{gather} \label{eq:aug_h}
    h_i(\state) = \mtx{ \frac{f_{x,i}}{m_i} &
    \frac{f_{y,i}}{m_i} &
    \dot{\theta}_i} \in \reals{3}.
\end{gather}
With an IMU at the CM of every link, the individual measurements can be concatenated and rearranged to form the overall measurement vector
\begin{gather}
    h(\state) = \mtx{{\massmat} \forces_x \\ 
    {\massmat} \forces_y \\
    \dot{\anglink}} = \mtx{\massmat (\forces_{ext, x} + \cmatadd \control) \\ \massmat (\forces_{ext, y} + \smatadd \control) \\ \dot{\anglink}} \in \reals{3N}.
\end{gather}

\subsection{State Transformation}
We will first apply a transformation to the state to simplify the Lie derivatives. Let the transformed state be \mbox{$\stateip' = \mathcal{T}(\stateip) \in \reals{4N}$} where
\vspace{-0.2cm}
\begin{align}
    \stateip' &= \mtx{\anglink & \tform & \mass & \inertia}
\end{align}
and $\tform = \inv{\mtheta} \inv{\inertiamat} \dot{\anglink}$. The transformation $\mathcal{T}(\stateip)$ is a diffeomorphism and if $\stateip'$ is observable, then $\stateip$ is observable. The measurement is rewritten in terms of $\stateip'$ by substituting the transformed state variable $\dot{\anglink} = \inertiamat \mtheta \tform$:
\begin{gather}
    h(\stateip') = \mtx{\massmat (\forces_{ext, x} + \cmatadd \control) \\ \massmat (\forces_{ext, y} + \smatadd \control) \\ \inertiamat \mtheta \tform} \in \reals{3N}.
\end{gather}
The Lie derivatives are taken with respect to the transformed state $\stateip'$ in the remainder of the analysis.

\subsection{Zeroth-Order Lie Derivative}
The zeroth-order Lie derivative is the Jacobian of the measurement function. The $*$ denotes non-zero blocks that are conservatively assumed to be zero, and thus contribute no rank and observability information, simplifying analysis.
\begin{align}
    \lie{h}{}{0} &= h(\stateip') \\
    \grad \lie{h}{}{0} &= \mtx{\mathbf{\Delta}_1 & \zeros{2N \cross N} & \mathbf{\Delta}_2 & \zeros{2N \cross N} \\
    * & \inertiamat \mtheta & * & *}  \\
    \mathbf{\Delta}_1 &= \mtx{-\massmat \smatadd \diag{\control} \\ \massmat \cmatadd \diag{\control}} \in \reals{2N \times N} \\
    \mathbf{\Delta}_2 &= \mtx{\diag{\forces_{ext,x} + \cmatadd \control} \\ \diag{\forces_{ext,y} + \smatadd \control}} \in \reals{2N \times N}.
\end{align}

\subsection{First-Order Lie Derivatives}
First-order Lie derivatives are taken with respect to an arbitrary control vector field $f_i$.
\vspace{-0.15cm}
\begin{align}
    \lie{h}{f_i}{1} &= \nabla h(\stateip') \cdot f_i(\stateip') \\
    &= \mtx{\zeros{2N} \\
    \inertiamat \lengthmat \left( \stheta \kmat \cmatadd - \ctheta \kmat \smatadd \right) \massmat \basis{i}} \\
    \grad \mathcal{L}_{f_{i}}^1 h &=
    \mtx{\zeros{2N \times N} & \zeros{2N \times N} & \zeros{2N \times N} & \zeros{2N \times N} \\
    * & \zeros{N \times N} & * & \mathbf{\Delta}_3} \\
    \mathbf{\Delta}_3 &= \diag{ \lengthmat \left( \stheta \kmat \cmatadd - \ctheta \kmat \smatadd \right) \massmat \basis{i}}.
\end{align}

\subsection{Observability Space}
By including all of the control vector fields in $\obsv{}$, we implicitly assume that $u_i \neq 0$. The first-order Lie derivative with respect to the drift vector field $f_0$ is excluded from the analysis because the derivatives became analytically intractable. Including $f_0$ terms may lead to less conservative conditions, i.e., requiring less control, for observability than those derived in this analysis.

Constructing $\obsv{}$ from the zeroth and first-order control Lie derivatives yields:
\begin{gather}
    \obsv{} = \mtx{\grad \lie{h}{}{0} \\ 
    \grad \lie{h}{f_{1:N}}{1}} \in \reals{(3N + 3N^2) \times 4N}.
\end{gather}
Compiling the blocks, removing zero block rows for simplicity, and swapping the first and third columns yields the diagonal block matrix $\obsv{}$ in \eqref{eq:obsv} at the top of the next page.
\begin{table*}[ht]
    \centering
    \begin{align} \label{eq:obsv}
        \obsv{} &= \mtx{-\massmat \smatadd \diag{\control} &  \diag{\forces_{ext,x} + \cmatadd \control} & \zeros{N \cross N} & \zeros{N \cross N} \\
        \massmat \cmatadd \diag{\control} & \diag{\forces_{ext,y} + \smatadd \control} & \zeros{N \cross N} & \zeros{N \cross N} \\
        \zeros{N \cross N} & \zeros{N \cross N} & \inertiamat \mtheta & \zeros{N \cross N} \\
        \zeros{N \cross N} & \zeros{N \cross N} & \zeros{N \cross N} & \diag{ \lengthmat \left( \stheta \kmat \cmatadd - \ctheta \kmat \smatadd \right) \massmat \basis{1}} \\
        \vdots & \vdots & \vdots & \vdots \\
        \zeros{N \cross N} & \zeros{N \cross N} & \zeros{N \cross N} & \diag{ \lengthmat \left( \stheta \kmat \cmatadd - \ctheta \kmat \smatadd \right) \massmat \basis{N}} }  \in \reals{(3N + N^2) \times 4N} 
    \end{align}
    \vspace{-0.8cm}
\end{table*}

The matrix in \eqref{eq:obsv} can be expressed in the compact block diagonal form as
\begin{gather}
    \obsv{} = \mtx{\mathbf{\Omega}_1 & \zeros{2N \times N} & \zeros{2N \times N} \\
    \zeros{N \times 2N} & \inertiamat \mtheta & \zeros{N \times N} \\
    \zeros{N^2 \times 2N} & \zeros{N^2 \times N} & \mathbf{\Omega}_2}.
\end{gather}
If each diagonal block of $\obsv{}$ is full rank, then the entire matrix is full rank \cite{Silvester2000-ww}. The block $\inertiamat \mtheta$ is trivially full rank because it is the product of positive definite matrices, so the task becomes proving that $\mathbf{\Omega}_1 \in \reals{2N \times 2N}$ and $\mathbf{\Omega}_2 \in \reals{N^2 \times N}$ are full rank. The remainder of this subsection focuses on deriving conditions for $\mathbf{\Omega}_1$ and $\mathbf{\Omega}_2$ to be full rank so that $\obsv{}$ is subsequently full rank.

\subsubsection{$\mathbf{\Omega}_1$ Analysis}
To prove the rank of $\mathbf{\Omega}_1$, we will utilize the following theorem:
\begin{theorem} \label{thm:schur}
Let  
\begin{gather}
    \mathbf{M} = \mtx{ \mathbf{A} & \mathbf{B} \\ \mathbf{C} & \mathbf{D} }
\end{gather}
be a block matrix where \( \mathbf{A}, \mathbf{B}, \mathbf{C}, \mathbf{D} \in \reals{N \times N}\) and \( \mathbf{C} \mathbf{D} = \mathbf{D} \mathbf{C} \). Then, the determinant of \( \mathbf{M} \) satisfies:  
\begin{gather}
    \det (\mathbf{M}) = \det (\mathbf{A} \mathbf{D} - \mathbf{B} \mathbf{C}) \ \cite{Silvester2000-ww}.
\end{gather}
\end{theorem}

Since the blocks of $\mathbf{\Omega}_1$ are all diagonal matrices and square, they commute with each other, and we can apply Theorem \ref{thm:schur} to compute the determinant of $\mathbf{\Omega}_1$.
Factoring out $-\massmat$ and applying the identity $\cmatadd^2 + \smatadd^2 = \eye{N}$, where $\eye{N}$ is the $N \times N$ identity matrix, yields
\begin{equation}
\begin{split}
    \det(\mathbf{\Omega}_1) &= \det(-\massmat) \det\big( \diag{\control^2} \\
    &\quad + \cmatadd \diag{\control} \diag{\forces_{ext,x}} \\
    &\quad + \smatadd \diag{\control} \diag{\forces_{ext,y}} \big). \label{eq:diag_sum}
\end{split}
\end{equation}
The term $\det(-\massmat)$ is always non-zero because $\massmat \in \diagpos{N}$, so it can be ignored for rank analysis purposes. In the other determinant term, the determinant is acting on a summation of diagonal matrices. The determinant of a diagonal matrix is the product of its diagonal entries, so the problem can be reduced to finding the conditions for each diagonal element to be non-zero. 

The diagonal entries of the diagonal matrix sum are
\begin{gather}
    u_i(u_i + \cos(\theta_i + \psi_i) f_{ext,x,i} + \sin(\theta_i + \psi_i) f_{ext,y,i}).
\end{gather}
We are implicitly assuming $u_i \neq 0$ by including the control vector fields in $\obsv{}$, so the element-wise condition for a non-zero determinant, and therefore a full rank $\mathbf{\Omega}_1$, becomes
\begin{gather}
    u_i + \cos(\theta_i + \psi_i) f_{ext,x,i} + \sin(\theta_i + \psi_i) f_{ext,y,i} \neq 0.
\end{gather}
This constraint implies that, in addition to requiring nonzero thruster forces on each link ($u_i \neq 0$), each link must also experience a nonzero net force.

\subsubsection{$\mathbf{\Omega}_2$ Analysis}
The block has the form of taking each column of $\lengthmat \left( \stheta \kmat \cmatadd - \ctheta \kmat \smatadd \right) \massmat$, diagonalizing it, and stacking the blocks. The block $\mathbf{\Omega}_2$ is full rank if $\lengthmat \left( \stheta \kmat \cmatadd - \ctheta \kmat \smatadd \right) \massmat$ does not have a row of zeros because the columns of $\mathbf{\Omega}_2$ will be linearly independent by construction.

Let $\fmat = \stheta \kmat \cmatadd - \ctheta \kmat \smatadd$. The diagonal structures of $\ctheta, \stheta, \cmatadd$, and $\smatadd$ allow the matrix products to be expressed in simple element terms where the subscript $ij$ represents the matrix element corresponding to the $i$-th row and $j$-th column. Applying the trigonometric identities
\begin{equation}
    \begin{aligned}\label{eq:trig}
        \sin(\alpha \pm \beta) &= \sin\alpha \cos\beta \pm \cos\alpha \sin\beta \\
        \cos(\alpha \pm \beta) &= \cos\alpha \cos\beta \mp \sin\alpha \sin\beta
    \end{aligned}
\end{equation}
and combining terms reduces the elements of $\fmat$ to
\begin{align}
    \fmat_{ij} = \kmat_{ij}  \sin((\theta_i - \theta_j) - \psi_j).
\end{align}

The matrix $\kmat$ has rank $N - 1$ and one can verify that $\ker(\kmat) =\{\mathbf{1}_N\}$ and \mbox{$\coker{\kmat} = \{\mtx{1 & -1 & 1 & -1 & \dots} \in \reals{N}\}$}.
From this information, we can conclude that $\kmat$ does not have a row or column of zeros because the basis vector $\basis{i}$ is not in either kernel subspace. We will use this information in the following theorem to prove that $\fmat$ does not have a row of zeros.

\begin{theorem}  
Let \(\kmat \in \reals{N \times N}\) be a matrix with \(\rank{\kmat} = N-1\) and no row or column composed entirely of zeros. Define the matrix \(\fmat\) with elements  
\begin{gather}
    \fmat_{ij} = \sin((\theta_i - \theta_j) - \psi_j) \, \kmat_{ij},
\end{gather}
where \(\sin(\psi_i) \neq 0 \ \forall \ i\). Then, \(\fmat\) cannot have a row consisting entirely of zeros.  
\end{theorem}
\begin{proof}
Since $\rank{\kmat} = N-1$ and $\kmat$ has no row or column of zeros, there exists a permutation matrix $\mathbf{P}$ that reorders the columns such that the permuted matrix $\kmat' = \kmat \mathbf{P}$ has all nonzero diagonal entries. Let $\fmat'$ be
\begin{gather}
    \fmat'_{ij} = \sin((\theta_i - \theta_j) - \psi_j) \, \kmat'_{ij}.
\end{gather}

We first show that $\fmat'$ cannot have a row of all zeros. The diagonal elements for any row $i$ are
\begin{gather}
    \fmat'_{ii} = \sin((\theta_i - \theta_i) - \psi_i) \, \kmat'_{ii} = -\sin(\psi_i) \, \kmat'_{ii}.
\end{gather}
Assuming $\sin(\psi_i) \neq 0$, and $\kmat'_{ii} \neq 0$ by construction, it follows that $\fmat'_{ii} \neq 0$ and therefore, no row of $\fmat'$ is entirely zero.

Since $\fmat = \fmat' \mathbf{P}^{-1}$ and $\mathbf{P}^{-1}$ is a permutation matrix, permuting the columns of $\fmat'$ does not introduce any zero rows. Therefore, $\fmat$ cannot have a row of all zeros.
\end{proof}

The matrices $\lengthmat$ and $\massmat$ are both diagonal positive definite and will not introduce any zero rows in the product $\lengthmat \fmat \massmat$, and therefore we can conclude that $\mathbf{\Omega}_2$ is full rank if $\sin(\psi_i) \neq 0 \ \forall \ i$.

We now present the main observability result:  
\begin{theorem} \label{thm:obsv}
    The augmented state $\stateip$ is locally observable under the following control conditions:  
    \begin{enumerate}  
        \item Non-zero thruster normal components: \\ \( u_i \sin(\psi_i) \neq 0 \ \forall \ i \).
        \item Non-zero net forces: \\ 
        $u_i + \cos(\theta_i + \psi_i) f_{ext,x,i} + \sin(\theta_i + \psi_i) f_{ext,y,i} \neq 0 \ \forall \ i.$
    \end{enumerate}  
\end{theorem}  
\begin{proof}  
    As shown in the previous subsections, both $\mathbf{\Omega}_1$ and $\mathbf{\Omega}_2$ are full-rank matrices. Since $\obsv{}$ is a diagonal block matrix with full-rank diagonal blocks, it follows that $\obsv{}$ itself is full-rank. Consequently, $\stateip'$ is locally observable, and since a diffeomorphism exists between $\stateip'$ and $\stateip$, it follows that $\stateip$ is also locally observable.  
\end{proof}
To the best of our knowledge, Theorem~\ref{thm:obsv} presents the first nonlinear observability result for a free-floating multi-link system that explicitly accounts for inertial parameter estimation. The theorem establishes sufficient control conditions under which the link angles, angular velocities, masses, and moments of inertia of system~\eqref{eq:multilink} can be estimated in real time using IMUs placed at the CM of each link. Although the analysis assumes a specific multi-link system, we hypothesize that similar external force and control conditions exist for multi-link systems with alternative actuation and sensing configurations. From a systems design perspective, the observability conditions require that the thrusters be oriented such that $\sin(\psi_i) \neq 0$ and have sufficient authority to counter external forces and produce a non-zero net force.

\section{SIMULATION RESULTS} \label{sec:sim}
As discussed in Section \ref{sec:nlobsv}, local observability means that, within a neighborhood of each state, the state can be uniquely reconstructed in finite time from measurements along trajectories that remain close to the initial condition. Practically, this implies that for a locally observable nonlinear system, an estimator such as a Kalman filter—when initialized with a mean and covariance sufficiently close to the true state—will yield estimation errors that converge to zero, with the estimate covariance approaching the Cramér–Rao lower bound. Estimator performance thus serves as a practical indicator of system observability and will be used here to validate the analytical results of the previous section.

A three-link system is simulated with state and dynamics according to \eqref{eq:aug_x} and \eqref{eq:aug_f}. The physical parameters are listed in Table \ref{table:parameters} and were selected to be comparable to those of the \emph{LandSalp} experimental platform. The inertia was calculated from the masses and lengths according to the equation for a slender rod rotating about one of its ends: $j_i = \frac{1}{3} m_i l_i^2$.
\begin{table}[ht]
    \centering
    \begin{tabular}{|c|l|c|} 
        \hline
        \textbf{Variable} & \textbf{Description} & \textbf{Value} \\ \hline
        $\mathbf{l}$ & Link half-lengths (m) & $0.125 \cdot \ones{3}$ \\ \hline
        $\mass$ & Link masses (kg) & $0.5 \cdot \ones{3}$ \\ \hline
        $\inertia$ & Link inertias (kg$\cdot$m$^2$) & $2.6 \times 10^{-3} \cdot \ones{3}$ \\ \hline
        $\angthrust$ & Thruster Angles (rad) & $\mtx{\frac{1}{4} \pi & \frac{2}{3} \pi & -\frac{1}{2} \pi}$\\ \hline
    \end{tabular}
    \caption{Simulated three-link system parameters.} \label{table:parameters}
    \vspace{-0.3cm}
\end{table}

Each link is equipped with an IMU at the origin of its respective link frame, and IMU measurements are generated at 100 Hz according to \eqref{eq:aug_h}. For simplicity, the model does not simulate any uncontrolled external forces (such as friction and drag). The input thruster signals were selected to mimic the pulsing control signals of salps and are of similar magnitude to the actuation ability of \emph{LandSalp}. The system was simulated for 1 second, and the thruster inputs consisted of unidirectional, periodic square waves that began at 0.2 seconds, turned on for 0.1 seconds, and then off for 0.2 seconds. The initial configuration of the system is at rest with parallel links ($\anglink_0 = \zeros{3}, \ \dot{\anglink}_0 = \zeros{3}$). Snapshots of the system trajectory at 0.2-second intervals are shown in Fig. \ref{fig:snapshots}. The thruster signals and a time history of the observability criteria from Theorem \ref{thm:obsv} are shown in Fig. \ref{fig:thrusters}.
\begin{figure*}[htbp]
    \vspace{0.05cm}
    \centering
    \includegraphics[width=0.83\textwidth]{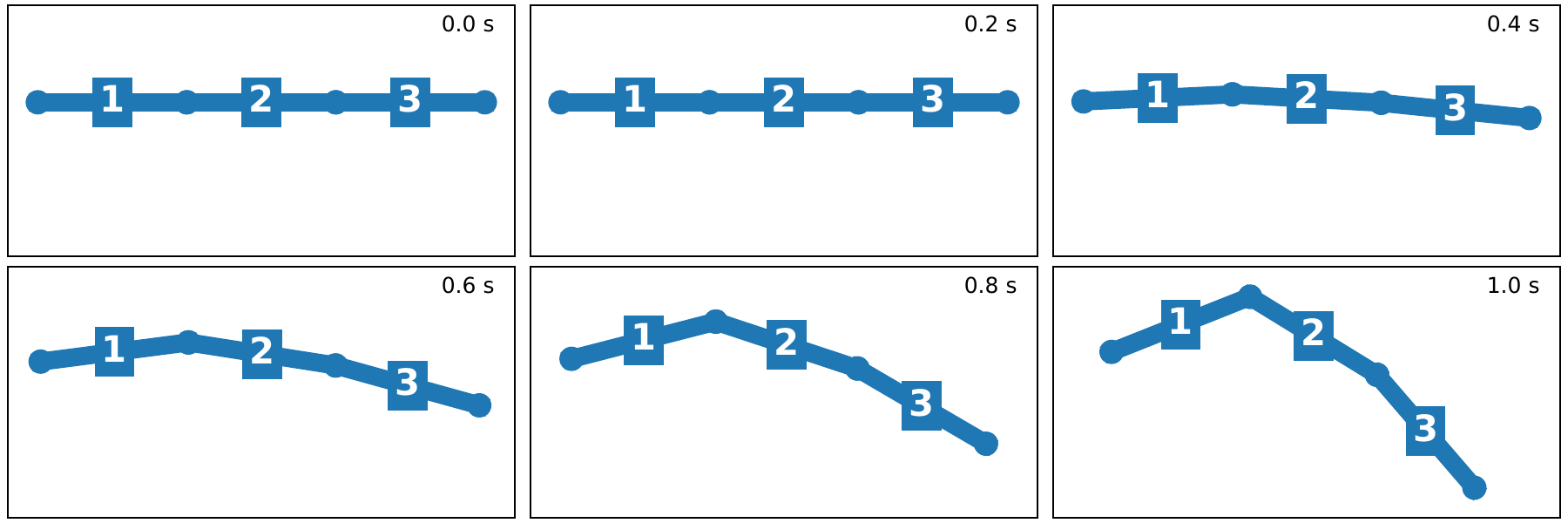}
    \caption{Snapshots of the simulated three-link system configuration at 0.2-second intervals.}
    \label{fig:snapshots}
    \vspace{-0.5cm}
\end{figure*}

\subsection{Filter Design}
An unscented Kalman filter was used for state estimation and developed using the FilterPy package \cite{filterpy}. The initial estimate mean, $\stateest_0$, was randomly selected according to \mbox{$\hat{\state}_0 \sim \mathcal{N}(\state_0, \pmat_0)$} where $\state_0$ is the true initial state and $\pmat_0$ is the initial estimate covariance as defined below. The $\blkdiag{\cdot}$ function denotes a block diagonal matrix with its inputs along the diagonal. The initial state estimate covariance values are shown in Table \ref{table:cov_p}.
\begin{align}
    \state_0 &= \mtx{\anglink_0 & \dot{\anglink}_0 & \mass & \inertia} \in \reals{12} \\
    \pmat_0 &= \blkdiag{\sigma_{\theta_0}^2 \eye{3}, \ \sigma_{\dot{\theta}_0}^2 \eye{3}, \ \sigma_{m_0}^2 \eye{3}, \ \sigma_{j_0}^2 \eye{3}} \label{eq:pmat} \in \diagpos{12}
\end{align}

\begin{table}[t]
    \centering
    \begin{tabular}{|c|l|c|} 
        \hline
        \textbf{Variable} & \textbf{Description} & \textbf{Value} \\ \hline
        $\sigma_{\theta_0}$ & Link angle (rad) & $0.1$ \\ \hline
        $\sigma_{\dot{\theta}_0}$ & Link angular velocity (rad/s) & $5.0 \times 10^{-3}$ \\ \hline
        $\sigma_{m_0}$ & Link mass (kg) & $0.1$ \\ \hline
        $\sigma_{j_0}$ & Link inertia (kg$\cdot$m$^2$) & $1.0 \times 10^{-4}$ \\ \hline
    \end{tabular}
    \caption{Initial estimate standard deviation values.} \label{table:cov_p}
    \vspace{-0.7cm}
\end{table}

Noisy IMU measurements on link $i$ are generated from the truth state and passed to the filter according to \mbox{$\meas_i \sim \mathcal{N}(h_i(\state), \rmat)$}, where $\rmat$ is the measurement noise covariance matrix for a single IMU and assumed to be the same for each of the simulated IMUs. Although the truth state propagation is noise-free in our simulation, a small amount of process noise was applied in the filter. This both reflects a qualitatively similar performance expected in practical applications and helps highlight differences in filter behavior when the observability conditions are satisfied versus when they are not. The process noise covariance matrix, $\qmat$, and measurement noise covariance matrix are defined in \eqref{eq:qmat}, \eqref{eq:rmat}, and Table \ref{table:cov_q}:
\begin{align}
    \qmat &= \blkdiag{\sigma_{\theta}^2 \eye{3}, \ \sigma_{\dot{\theta}}^2 \eye{3}, \ \sigma_m^2 \eye{3}, \ \sigma_j^2 \eye{3}} \label{eq:qmat} \in \diagpos{12} \\
    \rmat &= \diag{\sigma_\text{acc}^2, \ \sigma_\text{acc}^2, \ \sigma_\text{gyro}^2} \in \diagpos{3}  \label{eq:rmat}
\end{align}


\subsection{Tracking Performance}
Augmented state tracking errors and their corresponding covariance bounds are shown in Fig.~\ref{fig:estimate_error}. As expected, since the IMUs directly measure the angular velocity of their respective links, the angular velocity estimates converge immediately. For the remaining states, the estimation errors and covariance values decrease when the thruster observability conditions are satisfied, consistent with the analytical results. When these conditions are not met, the covariances of the unobservable states gradually increase due to the propagation of uncertainty through the nonlinear dynamics and the addition of process noise.
\begin{table}[t]
    \centering
    \begin{tabular}{|c|l|c|} 
        \hline
        \textbf{Variable} & \textbf{Description} & \textbf{Value} \\ \hline
        $\sigma_{\theta}$ & Link angle (rad) & $5.0 \times 10^{-3}$ \\ \hline
        $\sigma_{\dot{\theta}}$ & Link angular velocity (rad/s) & $1.0 \times 10^{-3}$ \\ \hline
        $\sigma_{m}$ & Link mass (kg) & $1.0 \times 10^{-2}$ \\ \hline
        $\sigma_{j}$ & Link inertia (kg$\cdot$m$^2$) & $7.0 \times 10^{-6}$ \\ \hline
        $\sigma_\text{acc}$ & Accelerometer (m/s$^2$) & $1.0 \times 10^{-2}$ \\ \hline
        $\sigma_\text{gyro}$ & Gyroscope (rad/s) & $1.0 \times 10^{-3}$ \\ \hline
    \end{tabular}
    \caption{Process and measurement noise values.} \label{table:cov_q}
    \vspace{-0.7cm}
\end{table}

A notable aspect of these results can be seen in the behavior of the inertia estimates. The accuracy of inertial parameter estimation depends on the degree of excitation: mass estimates require excitation through forces, while inertia estimates require excitation through torques. In this simulation, link~2 is hypothesized to be excited less than the other links, which is reflected in its comparatively larger inertia covariance. This qualitative behavior—that estimator performance is proportional to state excitation—has two important implications:
\begin{enumerate}
    \item It motivates the design of “observability-optimal” trajectories to minimize estimation covariance.
    \item It highlights the importance of accurate process and measurement models. Even theoretically observable states can become effectively inestimable if they are insufficiently excited or if their signals are masked by process and measurement noise.
\end{enumerate}
\vspace{-0.15cm}
\begin{figure*}
    \vspace{0.05cm}
    \centering
    \includegraphics[width=\linewidth]{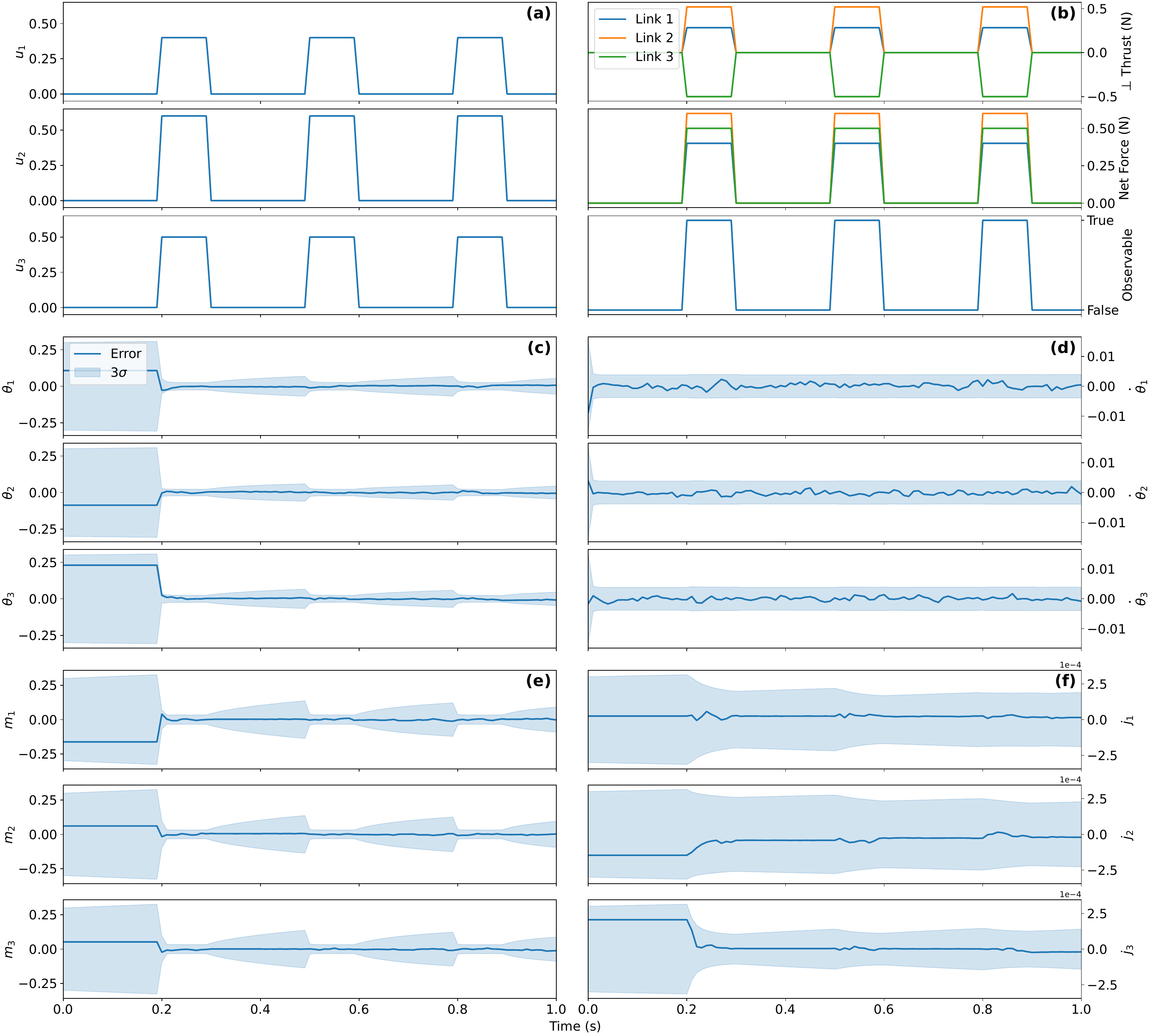}
    \caption{Simulation results: (a) thruster input forces (N); (b) observability conditions from Theorem~\ref{thm:obsv}; (c) estimation errors of link angles with $3\sigma$ bounds; (d) estimation errors of link angular velocities with $3\sigma$ bounds; (e) estimation errors of link masses with $3\sigma$ bounds; and (f) estimation errors of link inertias with $3\sigma$ bounds.}
    \label{fig:thrusters} \label{fig:estimate_error}
    \vspace{-0.5cm}
\end{figure*}

\section{CONCLUSIONS} \label{sec:conclusion}

In this work, we adapted a multi-link snake robot model to represent a chain of salp robots, derived the measurement equation for an IMU on an arbitrary link, and proved that the system’s link angles, angular velocities, masses, and moments of inertia are locally observable under the control conditions of Theorem~\ref{thm:obsv}. We demonstrated these conditions in simulation with a three-link system. Near-term work will focus on validating the analytical findings through hardware experiments using \emph{LandSalp}, while longer-term efforts will examine how observability is affected when modeling each link as a soft, flexible structure rather than a rigid body.

\vspace{-0.15cm}
\bibliography{references}

\end{document}